\documentclass[letterpaper]{article} 
\usepackage{aaai24}  
\usepackage{times}  
\usepackage{helvet}  
\usepackage{courier}  
\usepackage[hyphens]{url}  
\usepackage{graphicx} 
\usepackage{natbib}  
\usepackage{caption} 
\usepackage{algorithm}
\usepackage{algorithmic}
\usepackage{booktabs}

\usepackage{amssymb,amsthm,amsmath,amsfonts,bm}
\usepackage{thmtools}
\usepackage{thm-restate}
\usepackage{mathtools}
\newtheorem{theorem}{Theorem}[section]
\newtheorem{proposition}[theorem]{Proposition}

\newtheorem{definition}[theorem]{Definition}

\def\gD{{\mathcal{D}}}
\def\gS{{\mathcal{S}}}
\def\gT{{\mathcal{T}}}
\def\gL{{\mathcal{L}}}
\def\gX{{\mathcal{X}}}
\def\gY{{\mathcal{Y}}}
\def\gZ{{\mathcal{Z}}}

\DeclareMathOperator*{\argmax}{arg\,max}
\DeclareMathOperator*{\argmin}{arg\,min}

\usepackage{pifont}
\newcommand{\cmark}{\ding{51}}%
\newcommand{\xmark}{\ding{55}}%

\usepackage{amsmath,amsfonts,bm}
\usepackage{thmtools}
\usepackage{thm-restate}
\usepackage{mathtools}

\def\eqref#1{equation~\ref{#1}}

\def\1{\bm{1}}

\DeclareMathAlphabet{\mathsfit}{\encodingdefault}{\sfdefault}{m}{sl}
\SetMathAlphabet{\mathsfit}{bold}{\encodingdefault}{\sfdefault}{bx}{n}

\def\gD{{\mathcal{D}}}

\def\gI{{\mathcal{I}}}

\def\gL{{\mathcal{L}}}

\def\gS{{\mathcal{S}}}
\def\gT{{\mathcal{T}}}

\def\gW{{\mathcal{W}}}
\def\gX{{\mathcal{X}}}
\def\gY{{\mathcal{Y}}}
\def\gZ{{\mathcal{Z}}}

\newcommand{\R}{\mathbb{R}}

\usepackage{xspace}
\newcommand{\ours}[0]{\texttt{MISTS}\xspace}
\newenvironment{myquotation}{\setlength{\leftmargini}{0em}\quotation}{\endquotation}
\usepackage{newfloat}
\usepackage{listings}
\DeclareCaptionStyle{ruled}{labelfont=normalfont,labelsep=colon,strut=off} 
\floatstyle{ruled}
\newfloat{listing}{tb}{lst}{}
\floatname{listing}{Listing}
\title{Enhancing Evolving Domain Generalization through Dynamic \\ Latent Representations}
\author {
    Binghui Xie\textsuperscript{\rm 1},
    Yongqiang Chen\textsuperscript{\rm 1},
    Jiaqi Wang\textsuperscript{\rm 1},
    Kaiwen Zhou\textsuperscript{\rm 1},
    Bo Han\textsuperscript{\rm 2},
    Wei Meng\textsuperscript{\rm 1},
    James Cheng\textsuperscript{\rm 1}
}
\affiliations {
    \textsuperscript{\rm 1}The Chinese University of Hong Kong\\
    \textsuperscript{\rm 2}Hong Kong Baptist University\\
    \{bhxie21, yqchen, kwzhou, jqwang23, wei, jcheng\}@cse.cuhk.edu.hk, bhanml@comp.hkbu.edu.hk
}

\usepackage{bibentry}

\begin{document}

\maketitle

\begin{abstract}
Domain generalization is a critical challenge for machine learning systems. Prior domain generalization methods focus on extracting domain-invariant features across several \textit{stationary domains} to enable generalization to new domains. However, in \textit{non-stationary} tasks where new domains evolve in an underlying continuous structure, such as time, merely extracting the invariant features is \textit{insufficient} for generalization to the evolving new domains. Nevertheless, it is non-trivial to learn both evolving and invariant features within a single model due to their conflicts. To bridge this gap, we build causal models to characterize the distribution shifts concerning the two patterns, and propose to learn both \textit{dynamic} and \textit{invariant} features via a new framework called Mutual Information-Based Sequential Autoencoders (\ours). \ours adopts information theoretic constraints onto sequential autoencoders to disentangle the dynamic and invariant features, and leverage a domain adaptive classifier to make predictions based on both evolving and invariant information. Our experimental results on both synthetic and real-world datasets demonstrate that \ours succeeds in capturing both evolving and invariant information, and present promising results in evolving domain generalization tasks.
\end{abstract}

\section{Introduction}
\label{sec:intro}
Domain generalization (DG) is a critical challenge for machine learning systems that requires model to generalize beyond the assumption that training and testing data come from identical and independent distributions~\citep{bengio2019meta}.
To address the issue, most previous DG methods focus on extracting domain-invariant features across several \textit{stationary} source domains~\citep{CORAL,deep_DG,groupdro,irmv1}. Nevertheless, domains could also be non-stationary and evolve along with certain structures~\cite{Wang22,DBLP:conf/icml/Qin0L22,wildtime,gagnon2022woods}. For instance, banks assess whether a person is likely to default on a loan by examining factors such as income, career, and marital status. However, as society changes over time, it is desirable to take account of predictable trends along with the time \citep{bai2022temporal} (i.e., concept shift) when making predictions for new customers. Moreover, these factors, e.g., income and career types, will also change gradually due to social developments (i.e., covariate shift). Figure~\ref{fig:portrait} shows another example of historical images of US high school students. Besides temporal factors, the data collected can also evolve along with geometry, transformation, and other factors \citep{DBLP:conf/icml/Qin0L22}. 
The task of generalizing under such shifts, known as \textit{Evolving Domain Generalization} (EDG), involves training models on examples from sequential source domains to generalize well to evolving unseen target domains \cite{Wang22,DBLP:conf/icml/Qin0L22}.

To tackle EDG, a common belief is to inherit the spirit of domain invariant learning and to learn evolving domain \textit{invariant features} for generalizing to new domains \citep{gagnon2022woods}. In contrast, plentiful empirical studies challenged the position and showed that it is also needed to additionally learn the dynamic information which is useful for predicting evolving domain patterns \citep{Wang22,bai2022temporal}. These phenomena raise a challenging research question: 
\begin{myquotation}
\textit{What features do we need to learn for successful EDG?}
\end{myquotation}

\begin{figure}[t]
   \begin{center}
\includegraphics[width=0.46\textwidth]{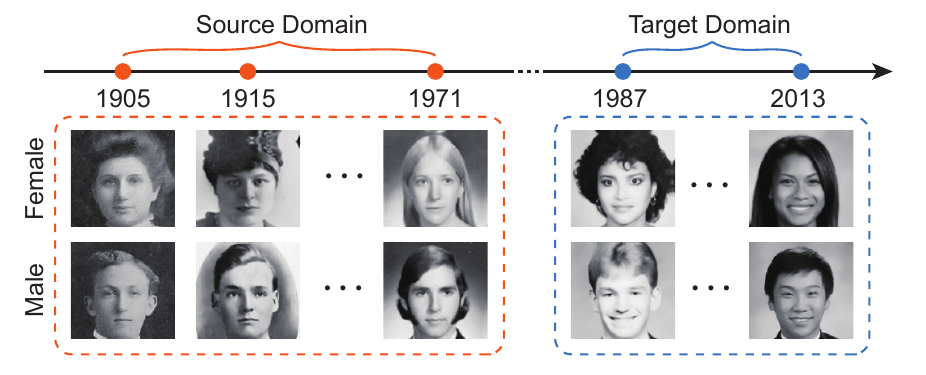}
  \end{center}
  \caption{An example of Evolving Domain Generalization on Portraits\citep{ginosar2015century}.The dataset consists of historical images of US high school students, and as time progresses, the visual attributes captured in the photos, such as hair type and clothing style, gradually change.}
  \label{fig:portrait}
\end{figure}
We consider the problem from the feature learning perspective. Specifically, we build Structural Causal Models to characterize the distribution shifts caused by the underlying invariant and dynamic factors in EDG.
Our causal analysis shows that it is essential to learn both dynamic and invariant features for better EDG.

Nevertheless, it is non-trivial to learn both
invariant and dynamic features within a single model due to their conflicts~\citep{DBLP:conf/icml/Qin0L22}.
To this end, we propose a principled EDG method called Mutual Information-Based Sequential Autoencoders (\ours). Specifically, \ours adopts a variational inference strategy to identify the underlying invariant and dynamic features. To encourage more complete separation between the invariant and dynamic latent representations, \ours employs a novel information-theoretic objective that minimizes the mutual information between them. We theoretically show the new objective of \ours is a valid evidence lower bound (ELBO) of the data log-likelihood with respect to our causal models. Our major contributions are as follows:
\begin{itemize}
    \item We provide theoretical evidence showing that either learning invariant or dynamic features is insufficient for EDG.
    \item We then propose a novel framework \ours to extract invariant and dynamic representations simultaneously and separately.
    \item We conduct extensive experiments on various EDG benchmarks. The results confirm that learning both invariant and dynamic features in \ours provides a better generalization ability on unseen evolving domains.
\end{itemize}

\section{Related Work}
\label{sec:related_work}
\subsection{Domain Generalization}
A rich literature is dedicated to addressing the OOD generalization challenge, which often involves introducing additional regularizations to Empirical Risk Minimization (ERM)~\citep{erm}.
Researchers such as \citet{DANN,CORAL,deep_DG,DouCKG19,chen2023pair} have explored regularization of learned features to be \textbf{domain-invariant}, while others such as \citet{dro,DRSL,groupdro,chen2023fat} have focused on regularizing models to be \textbf{robust to mild distributional perturbations} in the training distributions.
Similarly, researchers such as \citet{zhang2021causaladv,jtt,cnc,lisa,gala} have proposed improving robustness with additional assumptions.
Recently, there has been a growing interest in adopting causality theory \citep{causality,towards_causality} and introducing causal invariance to representation learning \citep{inv_principle,irmv1,env_inference,andmask,clove,ib-irm,chen2023does}.
These approaches require the learned representation to be \textbf{causally invariant}, such that a predictor acting on minimizes the risks of all environments simultaneously. 
In addition, approaches such as \citet{iga,vrex,fish,fishr} have implemented invariance by \textbf{encouraging agreements} at various levels across environments.
The aforementioned studies focus on extracting the invariant features across multiple domains while disregarding other features.
However, \citet{wildtime} provide extensive empirical evidence showing that the existing invariant learning methods may not be suitable for non-stationarity environments. This work focuses on OOD generalization under non-stationarity environments.

Besides, some methods are proposed to utilize domain-dependent features for better OOD generalization ability, e.g., \citep{chattopadhyay2020learning,zhang2021adaptive, bui2021exploiting,zhang2023aggregation}. Yet, they still treat domain index as a \textbf{discrete variable}, and can not learn the evolving drift across the domains.

\subsection{Evolving Domain Generalization}
Recently, many works have been proposed to tackle the challenging Evolving Domain Generalization or Temporal Domain Generalization, which can be further divided into two categories. The first line of work is primarily inspired by domain-dependent methods. For example, 
\citet{Wang22} propose to learn a transformation between domains via meta-learning. However, they assumed that the sequential domains \textbf{evolve consistently}, i.e., there exists an explicit function transformations between the domains, which does not always hold for real-world datasets. \citet{GI} introduce a temporal DG algorithm with gradient interpolation (GI) that trains models to predict near-future data by learning how the activation function evolves over time. Since GI only focuses on the change of activation functions, GI has \textbf{limited power} in characterizing model dynamics \citep{bai2022temporal}. Motivated by this, \citet{bai2022temporal} further propose to use dynamic graphs over a recurrent structure to capture the evolving dynamics of model parameter distributions. However, the method requires adjusting the model parameters by updating the weights of edges between neurons as a graph, which can be \textbf{computationally expensive} for large datasets and models. Meanwhile, \citet{DBLP:conf/icml/Qin0L22} propose LSSAE to model the underlying latent variables in data sample space. LSSAE disentangles the latent variables into invariant variables and dynamic variables. However, LSSAE still focuses on using the \textbf{invariant features} for prediction. In contrast, we incorporate a more rigorous analysis to demonstrate the importance of incorporating dynamic features for classification.

\subsection{Continuous Domain Adaptation}
The issue of continuous domain adaptation, or evolving domain adaptation, has garnered increasing attention in recent years. Various CDA methods have been developed, such as \cite{Hoffman2014CVPR,Wang2020CIDA,Long2020EDA,Lao2020Continuous}. Meanwhile, some intermediate-domain-based approaches \cite{Kumar2020ICML,Chen2021NIPS,wang2022understanding} are also known as gradual domain adaptation. However, these methods require data samples from target or intermediate domains for adaptation. Our focus is on the domain generalization task in evolving settings, where no information from target domains is accessible during training.

\section{Backgroud and Motivation}
\label{sec:motivation}

\begin{figure}[t]
  \begin{center}
   \includegraphics[width=0.3\textwidth]{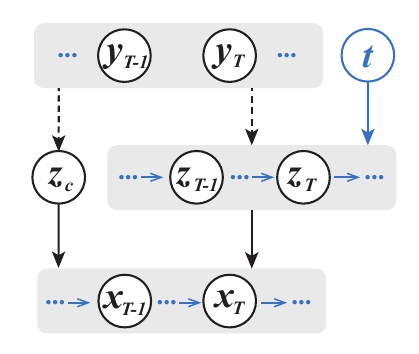}
  \end{center}
  \caption{The directed acyclic graph depicting our generative model. Dashed lines indicate the causal direction is possible for either side.}
  \label{fig:sem}
\end{figure}
\subsection{Problem Definition}
We consider the domain generalization tasks in which we have a sequence of evolving source domains $\mathcal{S}=\{ \mathcal{D}_1,\mathcal{D}_2,\dots,\mathcal{D}_T \}$, where each domain $\mathcal{D}_t=\{{(x_{i}^{(t)},y_{i}^{(t)})\}_{i=1}^{n_{t}}}$. $x_{i}^{(t)}\in\gX,y_{i}^{(t)}\in\gY$ and $n_{t}$ denote the input data, label and data size of domain $t \in \{1,2,\dots,T\},$ respectively. Our goal is to train a model $f_\theta:\gX\rightarrow\gY$ on source domains $\mathcal{S}$ to predict well on the evolving unseen target domains $\mathcal{T}=\{ \mathcal{D}_{T+1},\mathcal{D}_{T+2},\dots\}.$ For clarity, we will omit the index $i$ and $t$ for notations involving only a single data point or domain. EDG assumes that the domain distribution is changing following some sequential patterns~\citep{Wang22,DBLP:conf/icml/Qin0L22,bai2022temporal}.

\subsection{Feature Learning for Successful EDG}
\label{sec:prolem_def}
To study what features the model should learn for successful EDG, we first build the structural causal models to characterize the distribution shifts in EDG, shown in Figure.~\ref{fig:sem}.
The label $y \in \{1, -1\}$ is randomly sampled from the uniform distribution at the label space. Then, the label $y$ further controls the generation of the latent variable $z=[z_c,z_t]^T$with respect to domain $t$, where $z$ is composed of a domain invariant part $z_c$ and a dynamic part $z_t$:
\begin{equation}
\label{eq:problem}
    z_c \sim \mathcal{N}(y\cdot \mu_c,\,\sigma^{2}_c I), \quad z_t \sim \mathcal{N}(y\cdot \mu_t,\,\sigma^{2}_t I),
\end{equation}
where $\mu_c \in \mathbb{R}^{d_c}$ and $\mu_t \in \mathbb{R}^{d_t}$ are the mean of latent invariant and dynamic features, respectively. $\mu_t$ is generated conditioned on previous $\mu_{<t},$ and the condition relationship can be parameterized by neural networks. Furthermore, the latent features $z_c$ and $z_t$ control the generation of invariant and dynamic patterns of the input feature $x$ via an injective function of the latent features $z_c$ and $z_t$, i.e., $x=g(z_c,z_t)$.
Although we have presented the model as the distribution of $z_c$ and $z_t$ conditioned on $y$, the causal directions can be viewed either way. In general, DG frameworks suppose to learn a featurizer $\varphi:\gX\rightarrow\gZ$, such that there exists a classifier acting on $w:\gZ\rightarrow\gY$ such that:
\begin{equation}
\label{eq:risk}
	\ \sum_{t\in\gT}\gL_t(y, w\circ \varphi(x)),
\end{equation}
where $\gL_t$ represents logistic or 0-1 loss. Prior invariant methods are designed to achieve robust performance on target domains by ignoring non-invariant features, i.e., by using an invariant classifier $[2\mu_c / \sigma_c^2,0]^T$ in combination with $[z_c, 0]$ \citep{irmv1,rosenfeld2020risks}. However, in the context of EDG, non-invariant features have a meaningful correlation with the label. Restricting the model to $z_c$ will further limit its generalization ability in EDG. The theoretical results are informally presented in Theorem~\ref{thm:th1}. Details are deferred to the Appendix.

\begin{theorem} (Informal)
\label{thm:th1}
In the linear setting of Eq.~\ref{eq:problem}, for any domain $t$, there exists a classifier $w_t$ acting on $z_c$ and $z_t$ that achieves a lower risk than the optimal classifier $w_*$ acting on $[z_c,0]$.
\end{theorem}
Theorem~\ref{thm:th1} implies that additionally using dynamic features can achieve better OOD generalization ability, which explains the success of EDG using dynamic features~\citep{bui2021exploiting}. Instead of learning sole invariant or dynamic features, Theorem~\ref{thm:th1} demonstrates that it is essential to learn both features for successful EDG.

\subsection{Harnessing Dynamic Features}
To effectively utilize dynamic features, we aim to find a feature space that distinguishes between $z_c$ and $z_t$, enabling the model to learn the evolving pattern of $z_t$. One common approach is to employ a multi-classification head, as demonstrated in \citep{bui2021exploiting}, to capture invariant and domain-specific features, respectively. However, the multi-classification head treats domains as discrete indices, hindering the ability to learn dynamic patterns. Furthermore, this method struggles to cleanly separate $z_c$ and $z_t.$

In summary, there are several challenges in harnessing $z_t$. The first challenge is \textbf{how to capture $z_t$ and learn the evolving dynamics from $z_{1:T}.$} Drawing inspiration from the literature on Sequential Autoencoders, we consider employing a probabilistic framework that utilizes variational inference to identify the latent structures of $z_c$ and $z_t.$  However, a \textbf{purely sequential autoencoder fails to cleanly separate $z_c$ and $z_t$}, leading to inaccurate learning of the evolving dynamics. Therefore, we propose an information-theoretical regularizer to minimize the mutual information between $z_c$ and $z_t$. Additionally, \textbf{the conditional probability $P(Y|z_c, z_t)$ varies across different domains}, making it difficult for a stationary classifier to account for this drift. Therefore, to address these variations, we employ an adaptive classifier $w_t$ that operates on top of $C(z_c, z_t)$, where $C$ represents a combination function. In practical implementations, the combination function is often simplified by using concatenation.

\section{Method}
\label{sec:method}
In this section, we will introduce our proposed method Mutual Information-Based Sequential Autoencoders, which jointly extract \textit{both} the invariant and dynamic features and train domain-adaptive classifiers onto the extracted features to achieve better generalization in EDG.

\subsection{Probabilistic Modeling}
Specifically, we can define the following probabilistic generative model for the input data of all source domains as
\begin{equation}\label{eq:joint_xy}
\begin{split}
p&(x_{1:T},y_{1:T},z_c,z_{1:T}, w_{1:T}) \\
&= p(x_{1:T}, z_{1:T}, z_c) p(y_{1:T}, w_{1:T}| z_{1:T}, z_c).
\end{split}
\end{equation}
where x in Eq.~\ref{eq:joint_xy} separates the generation of $(x_{1:T},y_{1:T})$ into two parts via the decomposition of $z_{1:T}$ into the invariant feature $z_c$ and the dynamic feature $z_t$.
Modeling the first term $p(x_{1:T}, z_{1:T}, z_c)$ elicits proper disentanglement of $z_c$ and $z_t$ at the latent space. Modeling the second term $p(y_{1:T}, w_{1:T}| z_{1:T}, z_c)$ is essentially to leverage the disentangled features to predict the labels.
In what follows, we will detail how to model the two terms for better EDG.

\begin{figure*}[ht]
    \centering
	\includegraphics[width=0.9\textwidth]{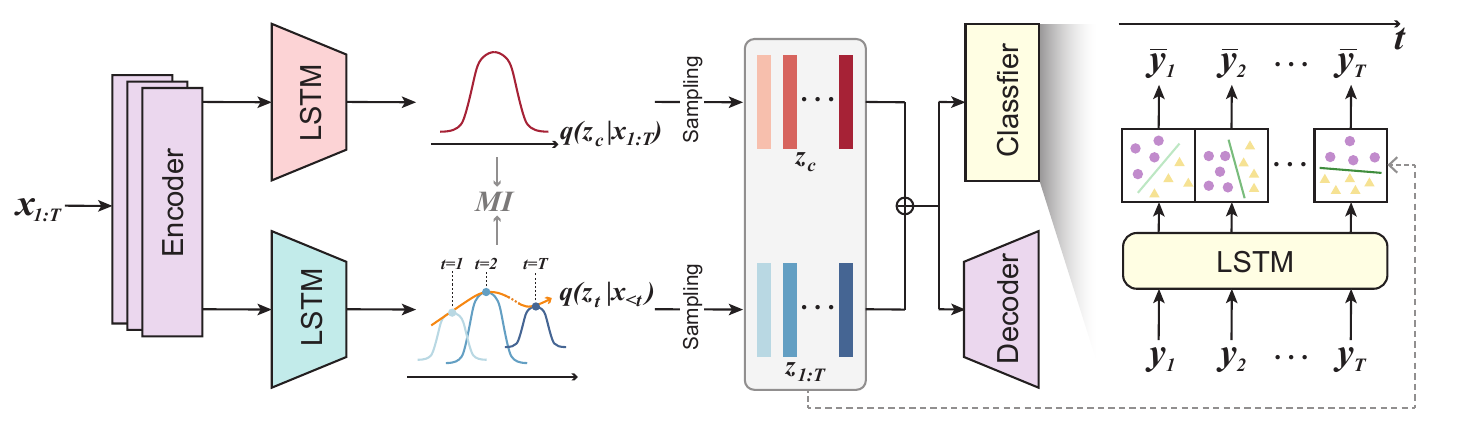}
	\caption{The \ours framework starts by encoding the input data into the latent space, where two distinct LSTMs parameterize the corresponding posteriors to obtain the invariant and dynamic latent representations, denoted as $z_c$ and $z_{1:T}$, respectively. These representations are then passed through a decoder and an adaptive classifier to compute the reconstruction loss and classification loss, respectively. To encourage better disentanglement, mutual information (MI) terms are applied to the invariant, dynamic latent variables, and input data. The KL-divergence terms are omitted for simplicity.}
 \label{fig:model}
\end{figure*}
\subsection{Identify Latent Representations}
\label{sec:encoder}
To model $p(x_{1:T}, z_{1:T}, z_c)$ for disentanglement, we can further decompose the first term by the chain rule of probability, as the following:
\begin{align}
\label{eq:gen_x}
    &p(x_{1:T}, z_{1:T}, z_c)=p(z)p(x_{1:T}|z)  \nonumber \\
    &=\bigg[p(z_c)\prod_{t=1}^T p(z_t|z_{<t})\bigg] \cdot \prod_{t=1}^T p(x_t|z_t, z_c).
\end{align}
Eq.~\ref{eq:gen_x} shows that the generation process of domains data $x_t$ at timestamp t depends on the corresponding dynamic latent representation $z_t$ and invariant representation $z_c.$ Moreover, the distribution of $z_t$ is conditional on $z_{<t}$ from the history domains. For instance, $z_c$ can represent the facial contour while $z_t$ denotes the visual attributes which follow fashion trends change over time. 

Our objective here is to extract the latent representations given only the observed data $x_{1:T}.$ Therefore, we hope to learn a posterior distribution as follows:
\begin{equation}
    q(z_{1:T},z_c|x_{1:T}) =
    q(z_c|x_{1:T}) \displaystyle\prod_{t=1}^T
    q(z_t|z_{<t},x_t),
\end{equation}
which has a corresponding variational lower bound as follows:
\begin{align}
\label{eq:vae}
\max_{p,q} \; \mathbb E_{q} \left[\log p(x_{1:T}|z)- KL[q(z|x_{1:T})||p(z)] \right],
\end{align}
where $p$ and $q$ can be parameterized by recurrent neural network, and $KL$ denotes KL divergence. The first term denotes the reconstruction term for input data $x_t$, the second and third terms denote KL divergence which is to align the posterior distributions $z_c$ and $z_t$ with the corresponding prior distributions.

Although forcing the dynamic variable $z_t$ to predict dynamic factors can guarantee that $z_t$ contains adequate dynamic information, it fails to guarantee that $z_t$ excludes the static information $z_c$. For example, if the learned dynamic factors $\hat z_t$ encompasses $z_c,$ the invariant factor $\hat z_c$ will lose some invariant information while $(\hat z_c, \hat z_{1:T})$ still reconstructs the data samples equally well as $(z_c, z_{1:T})$ up to Eq.~\ref{eq:vae}. Consequently, such disentanglement results can lead to poor performance in the downstream classification tasks \citep{ilse2020diva,DBLP:conf/icml/Qin0L22}. To solve the problem, \citet{DBLP:conf/icml/Qin0L22} used a temporal domain constraint to limit the information carried by $z_t.$ However, the constraints will harm the encoding of $z_t,$ and fail to capture complex dynamics. 

From the information-theoretic perspective, the problem can be considered as clearly separating $z_c$ and $z_t$, under the constraint that $z_t$ and $z_c$ should contain useful information from $x_t$ in a specified domain \citep{han2021disentangled,akuzawa2021information,DBLP:conf/nips/BaiWG21}. Specifically, we propose to additionally minimize the mutual information between $z_t$ and $z_c$ while maximizing the mutual information between the latent representation and the data observed, such that the $z_t$ and $z_c$ carry mutually exclusive meaningful information on input data:
\begin{equation}
\label{eq:vae_mi}
\begin{aligned}
\max_{p,q} \; & \mathbb E_{q} \left[\log p(x_{1:T}|z)- KL[q(z|x_{1:T})||p(z)] \right] \\
 \text{s.t.}\; &p,q\in \argmin I_q (z_t;z_c),\\
&p,q\in \argmax  I_q (z_;x_t),
\end{aligned}
\end{equation}
where the mutual information  terms are defined as:
\begin{align}
\label{eq:mi_term}
    I_q (z_c;x_t)&=\mathbb E_{q(z_c, x_t)} \left[\log q(z_c|x_t) - \log {q(z_c)}\right],\nonumber  \\
    I_q (z_t;x_t)&=\mathbb E_{q(z_t, x_t)} \left[\log q(z_t|x_t) - \log {q(z_t)}\right].
\end{align}
$I_q (z_t;z_c)$ is defined similarly, and more details can be found in the Appendix. To address the constrained optimization problem stated in Equation \ref{eq:vae_mi}, we demonstrate that it can be transformed into the optimization of a novel valid Evidence Lower Bound (ELBO) for the data log-likelihood. We provide a theoretical guarantee for this transformation, ensuring the reliability of the approach.

\begin{theorem}
\label{thm:th2}
Let the mutual information (MI) between $z_c,$ $z_t$ and $x_t$ be Eq.~\ref{eq:mi_term} in terms of $q,$ and $p(z_t)=p(z_t|z_{<t}, x_t),$ then solving Problem.~\ref{eq:vae_mi} can be converted into an evidence lower bound of $\mathbb E_{x_{1:T}\sim p_D} \log (x_{1:T})$ as follows:

\begin{align*}
\max_{p,q} \; \mathbb E_{q} \left[\log p(x_{1:T}|z)- KL[q(z)||p(z)] \right],
\end{align*}

\end{theorem}
In practice, we follow the spirit of \citep{higgins2016beta,DBLP:conf/nips/BaiWG21,DBLP:conf/icml/Qin0L22} to optimizating the following  objective:
\begin{align}
\label{eq:elbo_mi}
\mathcal{L}_e
=\; & \displaystyle\sum_{t=1}^T\mathbb E_{z_c, z_t}\big[\log p(x_t|z_c, z_t) \nonumber \\ 
&- \alpha (KL[q(z_c|x_{1:T})||p(z_c)] \nonumber \\ &+ KL[q(z_t|z_{<t}, x_{t})||p(z_{t}|p(z_{<t})]) \nonumber \\  &+ \beta  (I_q (z_c;x_t) + I_q (z_t;x_t) - I_q (z_t;z_c))\big],
\end{align}

\subsection{Domain Adaptive Classifier}
\label{sec:classifier}
To model $p(y_{1:T}, w_{1:T}| z_{1:T}, z_c)$ for classification, as discussed in Sec.~\ref{sec:motivation}, the classifier is expected to evolve accordingly, e.g. the magnitude of model parameter weights for some features will evolve gradually.
We use $w_t$ to represent the classifier for domain $t$ and consider $w_t$ as a latent variable in the category space, inspired by \citet{DBLP:conf/icml/Qin0L22}. 
According to Eq.~\ref{eq:joint_xy}, the adaptive classifier $w_t$ can be decomposed as:
\begin{align*}
\label{eq:gen_y}
    p(y_{1:T}, w_{1:T}| z_{1:T}, z_c) = \prod_{t=1}^T
    p(w_{t}|w_{<t})p(y_{t}|z_t, z_c, w_{t}),
\end{align*}
where $w_t$ can be inferred by classifiers from previous domains, and predict labels $\hat y_t$ given $z_c$ and $z_t.$ 
To model the dynamic variables $w_t$, we parameterize $p({w}_{t}|{w}_{<t})$ as a learnable categorical distribution.
To optimize the prior distribution $p({w}_{t}|{w}_{<t}),$ the objective function can also be derived based on ELBO:
\begin{equation}
\label{equ:loss_c}    
\begin{split}
    \mathcal{L}_c &= \lambda \displaystyle\sum_{t=1}^T\mathbb{E}_{q(w_t|w_{<t},{y}_t)} \big[\log p(y_{t}|z_c,z_{t},w_t)  \\
    &- \alpha KL[ (q(w_t|w_{<t},y_t),p(w_{t}|w_{<t}))]\big],
\end{split}
\end{equation}
where $q({w}_t|{w}_{<t}, {y}_t)$ denotes the posterior distribution and $p(y_{t}|z_c,z_{t},w_t)$ denotes the classification loss. Given training data $\gS$, our proposed framework can be optimized through the objective function $\mathcal{L}_{\ours}=\mathcal{L}_e + \mathcal{L}_c.$ To guranttee the efficacy of our method proposed, we provide the following proposition:
\begin{proposition}
\label{prop}
Given the probabilistic generative model defined for the joint distribution of all source domains as in Eq.~\ref{eq:joint_xy}, $\mathcal{L}{_\ours}$ is equivalent to the ELBO of the data log-likelihood $\log (x_{1:T}, y_{1:T})$ on source domains.
\end{proposition}
Proposition~\ref{prop} demonstrates that \ours can effectively identify the latent representations of interest and model their relationship with the labels within a given dataset.

\section{Model Instantiations and Implementation}
\label{sec:imp}
The implementation of network architecture for \ours is depicted in Fig.~\ref{fig:model}. It is composed of two parts: (1) encoders to extract $z_c$ and $z_{1:T}$ (2) adaptive classifier on top of the two latent representations. 
\paragraph{Encoder.}The encoder module consists of a base-encoder $\theta$ to learn latent representation from the data samples $x_t$, and two LSTMs $\kappa_c$ and $\pi_q$ to extract $z_c$ and $z_t$ based on $q(z_c|x_{1:T})$ and $q(z_t|z_{<t}, x_t)$, respectively. We set the corresponding prior distribution $p(z_c)$ to be the standard Gaussian, and $p(z_t|z_{<t})$ is instantiated by $\pi_p$, which share the same architectures with $\pi_q.$ 
\paragraph{Decoder.} The decoder $D$ takes the concatenation of $z_c$ and $z_t$ as input and output the reconstructed data samples $\hat x_t$. For the classification module, the prior network $\tau_p$ for $p(w_t|w_{<t})$ is a LSTM network with a categorical distribution as the output, which generates linear classifiers at domain $t$ based on the previous sequence $w_{t}.$ The corresponding posterior network $\tau_q$ for $q(z^v_t|z_{<t}^v,y_t)$ share a similar structure while additionally taking the one-hot code of label $\mathbf{y}_t$ as the input. Finally, we can generate the classification results $\bar{y_t}$ by using $w_t$ and the combination of $z_c$ and $z_t.$ The detailed optimization procedure can be found in Algorithm~\ref{alg:ours}.

\noindent{\textbf{Inference.}}~ To predict the label of $x_{T+1}$ sampled from the following target domains in $\gD_{T+1}$, we adopt $\pi_q$ and $\tau_p$ to infer $z_{T+1}$ and $w_{T+1}$ and use $\kappa_c$ to extract the latent representation $z_c$. Then, we can have the prediction results $\overline{y}_{T+1}.$ Reapting this process, we can generalize our model into future target domains, e.g., $\gD_{T+2}, \gD_{T+3}.$
\begin{algorithm}[t]
\caption{Optimization procedure of \ours}
\label{alg:ours}
    \begin{algorithmic}[1]
        \STATE {\bfseries Input:} Source labeled datasets $\gS$ with $T$ domains; Training epochs $E;$ Batch Size $B.$
   \STATE Randomly initialize $\theta, \kappa_c, \pi_p, \pi_q, \tau_p, \tau_q, D$
   \STATE Assign $\mathbf{z}_0, \mathbf{w}_0 \gets \mathbf{0}$

   \FOR {$t = 1,2,\dots,E$}
      \FOR {$i = 1, 2,\dots,T$}
      \STATE Sample a batch $B$ of data $(x_t, y_t)$ from $\gD_t$
      \STATE Compute batch loss $\mathcal{L}_e$ according to Eq.~\ref{eq:elbo_mi} for $\kappa_c, \pi_p, \pi_q$ and $D$
      \STATE Compute batch loss $\mathcal{L}_c$ according to Eq.~\ref{equ:loss_c} for $\tau_p, \tau_q$

      \STATE Update all modules by the total loss $\mathcal{L}_{\ours}$
      \ENDFOR
   \ENDFOR
    \end{algorithmic}
\end{algorithm}

\section{Experiments}
\label{sec:exp}

\subsection{Experimental Setup}
To evaluate the effectiveness of \ours, we conducted experiments on both synthetic and real-world datasets, following the setting of LSSAE \citep{DBLP:conf/icml/Qin0L22}. Specifically, we compared our approach with invariant learning methods and the state-of-the-art EDG methods on three synthetic datasets (Circle, Sine and Rotated MNIST) and three real-world datasets (Portraits, Caltran, and Elec). We also evaluated the results on one additional variant, Sine-C, which was created for EDG settings by \citet{DBLP:conf/icml/Qin0L22}. The domains were split into the source, intermediate, and target domains with a ratio of ${1/2:1/6:1/3}$, with the intermediate domains used as the validation set. 

\paragraph{Datasets.}We briefly introduce these datasets here, and leave more details on the dataset and baselines in the Appendix. (1) The \textbf{Circle} dataset~\cite{Pesaranghader2016FastHD} includes 30 evolving domains, where data points are sampled from 30 2D Gaussian distributions. For \textbf{Circle-C}, concept shift is introduced by gradually changing the center and radius of the decision boundary over time. (2) The \textbf{Sine} dataset~\cite{Pesaranghader2016FastHD} is extended to 24 evolving domains by rearranging it. To test models whether can adapt to sudden change, the labels of \textbf{Sine-C} are reversed (i.e., from 0 to 1 or from 1 to 0) from the 6th domain to the last one. (3) The \textbf{Rotated MNIST} (RMNIST) dataset~\cite{Ghifary2015RMNIST} consists of MNIST digits with varying degrees of rotation. (4) The \textbf{Portraits} dataset~\cite{ginosar2015century} contains photos of high-school seniors from the 1900s to the 2000s for gender classification. We split the dataset into 34 domains based on a fixed interval over time. (5) The \textbf{Caltran} dataset~\cite{Hoffman2014CVPR} is a real-world surveillance dataset comprising images captured from a fixed traffic camera deployed in an intersection. It involves predicting the type of scene based on continuously evolving data. We divide it into 34 domains based on different times. (6) The \textbf{Elec} dataset~\cite{Dau2019UCR} is designed for the time-section prediction of current power supply based on the hourly records of an Italian electricity company. The concept shift may arise from changes in season, weather, or price. We split it into 30 domains based on days.

\paragraph{Baselines.}We compare \ours with the state-of-the-art EDG methods, i.e., LSSAE and DRAIN. Additionally, we choose several representative methods from 3 main categories: classical supervised learning, continual learning, and invariant learning. \textbf{Classical Supervised Learning:} ERM~\citep{erm}. \textbf{Continual Learning:} EWC~\citep{kirkpatrick2017overcoming} and SI~\citep{zenke2017continual}. \textbf{Invariant learning}: (5) IRM~\citep{irmv1}, CORAL~\citep{CORAL}, Mixup~\citep{Yan2020Mixup} and LISA~\citep{lisa}. We  implemented these baselines according to \citep{wildtime}.

\begin{table*}[!t]

\begin{center}
\begin{normalsize}
\begin{tabular}{c|cccccccc}
\toprule
\textbf{Algorithm} & \textbf{Circle} & \textbf{Sine}  & \textbf{Sine-C} & \textbf{Elec} & \textbf{RMNIST} &  \textbf{Portraits} & \textbf{Caltran}  & \textbf{Avg}\\
\midrule
ERM  & 49.3  $\pm$  2.1   & 62.7 $\pm$ 1.1 & 62.3 $\pm$ 1.2 & 70.8 $\pm$ 0.6 & 41.6 $\pm$ 0.7 & 87.9 $\pm$ 1.4 & 61.2 $\pm$ 2.3 &  62.3    \\
IRM   & 53.6 $\pm$ 2.7  & 62.3 $\pm$ 0.9 & 59.4 $\pm$ 0.7 &70.5 $\pm$ 0.1   & 40.2 $\pm$ 0.4 & 87.2 $\pm$ 1.7 & 63.8 $\pm$ 0.8 &  62.4  \\
Mixup  & 48.7 $\pm$ 1.5  & 62.2 $\pm$ 0.7 & 62.0 $\pm$ 1.5 &70.2 $\pm$ 0.9  & 42.1 $\pm$ 0.4 & 87.8 $\pm$ 1.0 & 66.1 $\pm$ 1.0  & 62.7   \\
LISA  & 48.4 $\pm$ 1.1   & 61.7 $\pm$ 0.5  & 61.3 $\pm$ 1.2  & 70.1 $\pm$ 1.0 & 41.9 $\pm$ 0.8 & 88.0 $\pm$ 0.9 & 66.3 $\pm$ 1.1  & 62.5  \\
CORAL  &51.2 $\pm$ 4.2  & 58.7 $\pm$ 1.9  & 60.2 $\pm$ 2.1  &70.1 $\pm$ 0.7 & 42.3 $\pm$ 0.8 & 86.1 $\pm$ 1.9 & 65.1 $\pm$ 1.5 &  62.0  \\
GroupDRO  & 54.3 $\pm$ 3.4  & 59.3 $\pm$ 0.2  & 59.2 $\pm$ 4.1  &68.6 $\pm$ 0.6  & 42.8 $\pm$ 1.1 & 84.3 $\pm$ 1.8 & 63.6 $\pm$ 1.3 & 61.7\\
\midrule
EWC  & 59.1 $\pm$ 3.9  & 65.4 $\pm$ 4.8  & 64.2 $\pm$ 4.3 & 69.5 $\pm$ 1.1  & 32.2 $\pm$ 5.7 & 88.7 $\pm$ 1.2 & 60.1 $\pm$ 3.2 & 62.7 \\
SI   & 58.3 $\pm$ 2.7 & 69.7 $\pm$ 3.9 & 63.6 $\pm$ 4.7  & 69.8 $\pm$ 1.3 &   31.4 $\pm$ 5.3 & 88.7 $\pm$ 0.8 & 59.4 $\pm$ 3.7  & 63.0  \\
\midrule
DRAIN  & 54.5 $\pm$ 4.0 & 69.5 $\pm$ 2.7 & \textbf{71.0} $\pm$ \textbf{1.3}  & 71.1 $\pm$ 0.7 & 44.2 $\pm$ 1.1 & - & -    & - \\
LSSAE  & \textbf{63.5} $\pm$ \textbf{4.7} & 68.4 $\pm$ 3.5 & 63.6 $\pm$ 3.7   & 71.0 $\pm$ 0.5 & 45.3 $\pm$ 1.4 & 88.9 $\pm$ 1.5 & 68.8 $\pm$ 3.4    & 67.0  \\
\textbf{MISTS}  & 62.7 $\pm$ 2.7 & \textbf{78.2} $\pm$ \textbf{3.3} & 70.1 $\pm$ 3.2   & \textbf{71.4} $\pm$ \textbf{0.8} & \textbf{47.5} $\pm$ \textbf{1.3} & \textbf{89.2} $\pm$ \textbf{1.3} & \textbf{70.2} $\pm$ \textbf{2.1} &  \textbf{69.9} \\

\bottomrule
\end{tabular}
\end{normalsize}
\end{center}
\caption{Experimental Results (Accuracy \%) on Synthetic and Real-World Datasets across Different Methods.}
\label{tab:results}
\end{table*}

\paragraph{Quantitative results.}The results of our proposed \ours and the baseline methods are presented in Table~\ref{tab:results}. The experimental results reveal the significant improvement of EDG (Evolving Domain Generalization) methods over traditional DG (Domain Generalization) methods. This finding aligns with our theoretical results as well as the empirical observations in the existing EDG literature. Additionally, the substantial performance gap observed between continual learning methods (such as SI and EWC) and EDG methods highlights the importance of effectively leveraging historical knowledge to learn evolutionary patterns, which is not taken into account by methods like SI and EWC.

Notably, \ours outperforms other EDG methods in terms of average accuracy across the seven datasets. These results highlight the importance of a clear separation of dynamic and invariant features and incorporating both features for successful EDG. By considering the evolving nature of the data, \ours demonstrates improved performance and shows promise for enabling better adaptation to changing environments from a feature learning perspective.

\begin{figure}[t]
  \begin{center}
   \includegraphics[width=0.46\textwidth]{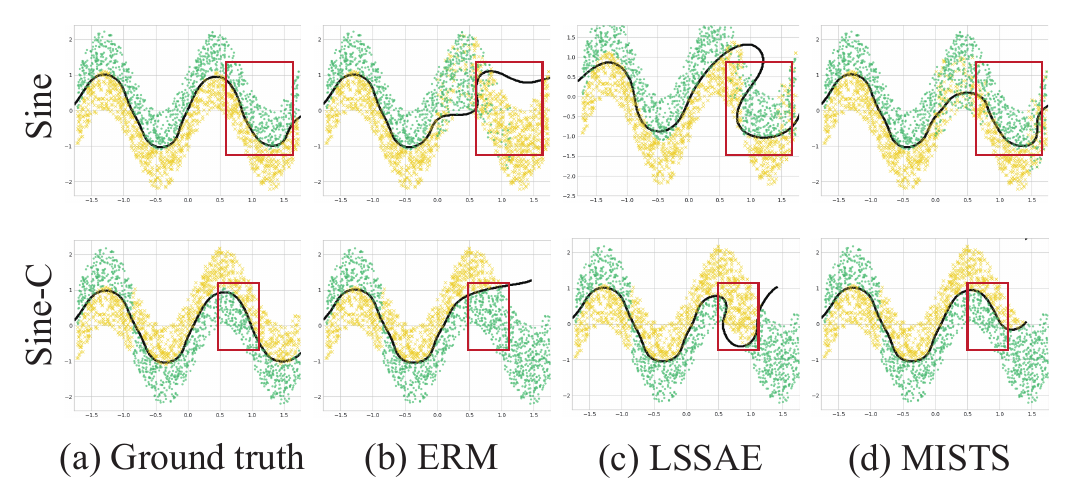}
  \end{center}
  \caption{The visualization presents the decision boundaries for the Sine and Sine-C datasets. In the Sine dataset's ground truth, positive and negative labels are denoted by green and yellow dots, respectively. Figures (b-d) illustrate the prediction results on the Sine dataset, obtained through the ERM, LSSAE, and MISTS methods, respectively.}
  \label{fig:vis}
\end{figure}

\paragraph{Qualitative results.}To assess the effectiveness of our method, we visualize the decision boundaries of our proposed approach along with two baselines, ERM and LSSAE, on the Sine and Sine-C datasets. The visualization results are presented in Figure~\ref{fig:vis}. For better visualization, we apply certain smoothing and augmentation techniques to the challenging Sine-C curve, which does not affect our results.

As depicted in the figure, all methods demonstrate a good fit to the source domains (the left half of the figures). However, unlike ERM, which only adapts to the source domains, both LSSAE and \ours exhibit desirable generalization capabilities to unseen target domains. This observation validates the effectiveness of our EDG approach in capturing underlying patterns that evolve across domains, resulting in improved performance. It is worth noting that LSSAE struggles to recover part of the desired decision boundary at the unobserved timestamps compared to \ours. This limitation may stem from LSSAE's sole reliance on an adaptive classifier built on top of invariant features, potentially leading to a loss of dynamic information. Additionally, all methods show potential for improvement when faced with abrupt changes.

Overall, the visualization results highlight the superiority of \ours in terms of generalization to unseen target domains, confirming its ability to capture and leverage evolving patterns for enhanced performance.

\subsection{Ablation Study}
In this subsection, we conduct ablation studies to assess the individual contributions of different components in \ours, using the RMNIST dataset. We systematically remove one or several components during training to create new variants of \ours and evaluate their performance. The results are as summarized in Table~\ref{tab:ablation}.

To begin, we investigate the significance of leveraging additional information for improved generalization. The results from variants A and B indicate a significant degradation in performance compared to the original \ours. This observation highlights the indispensability of both $z_c$ and $z_t$ components for achieving promising results.

Furthermore, we explore the impact of clean separation on the efficiency of utilizing $z_c$ and $z_t$. Variant C exhibits a significant performance drop compared to Variant A due to the absence of Mutual-Information terms, suggesting that a clean separation is crucial for optimal utilization of these components. Conversely, Variant D performs slightly better than the original \ours, indicating that it can leverage some invariant information during the extraction of dynamic features, without a clean separation. Finally, Variant E demonstrates that the absence of an adaptive classifier leads to performance similar to that of the ERM (Empirical Risk Minimization) baseline. This finding suggests that the adaptive classifier plays a vital role in enhancing the model's performance, distinguishing it from the conventional ERM approach.

Overall, these ablation studies provide valuable insights into the importance and interplay of different components within \ours, shedding light on the factors that contribute to its superior performance compared to the variants and baselines examined.

\begin{table}[t]

    \begin{center}
\begin{normalsize}
            \begin{tabular}{c|c|c|c|c|c}
                \hline
                             & $\ z_c\ $ & $\ z_t\ $ & $\ w_t\ $ & MI & Accuracy              
                \\\hline
                Variant A     & \cmark   & \xmark  &  \cmark  &  \cmark  &  45.7$\ \pm\ $1.1  \\\hline
                Variant B    & \xmark   & \cmark  & \cmark & \cmark  & 10.2$\ \pm\ $0.2  \\\hline
                Variant C & \cmark  & \xmark  & \cmark &  \xmark  &  44.3$\ \pm\ $1.4 \\
                \hline
                Variant D & \xmark  & \cmark  & \cmark &  \xmark  &  11.4$\ \pm\ $0.4   \\
                \hline
                Variant E & \cmark  & \cmark  & \xmark &  \cmark  &  42.1$\ \pm\ $1.3   \\
                \hline
                \textbf{MISTS} & \cmark  & \cmark  & \cmark &  \cmark  &  47.5$\ \pm\ $1.3   \\
                \hline
            \end{tabular}
        
\end{normalsize}
    \end{center}
    \caption{Ablation study of \ours on dataset RMNIST.}
        \label{tab:ablation}
\end{table}

\section{Conlusion}

This paper introduces a novel approach, Mutual Information Based Sequential Autoencoders (\ours), for addressing the challenges of Evolving Domain Generalization (EDG). Our empirical findings demonstrate the significant impact of \ours on enhancing the performance of downstream classification tasks within the EDG framework. The theoretical analysis further supports the effectiveness of our method in modeling the joint distribution of data and labels under evolving distribution shifts.

\paragraph{Future Work.} While our work makes significant contributions to the understanding and advancement of Evolving Domain Generalization (EDG), it is crucial to acknowledge its limitations. One important aspect to consider is the varying importance of invariant and dynamic features across different datasets as their distributions change. Therefore, it becomes essential to find a method that can automatically adapt and adjust the relative importance of these features.

Overall, As the first to focus on simultaneously learning dynamic and invariant features, we aspire to inspire further research in exploring the essential factors that impact the performance of EDG models and effectively leveraging these factors.

\section*{Acknowledgements}
We thank reviewers and meta-reviewers for their valuable comments. This work was supported by CUHK direct grant
4055146. BH was supported by the NSFC Young Scientists Fund No. 62006202, NSFC General Program No. 62376235, Guangdong Basic and Applied Basic Research Foundation No. 2022A1515011652, and HKBU Faculty Niche Research Areas No. RC-FNRA-IG/22-23/SCI/04.

\bibliography{aaai24}

\onecolumn

\section{Derivations of Section~\ref{sec:motivation}}
\label{sec:appendixA}
In this section, we give the formal proof of Theorem~\ref{thm:th1} and the derivation to show that the  causal directions of Eq.~\ref{eq:problem} can be
viewed either way. For completion, we first introduce the Invariant Risk Minimization (IRM), which is one of the representative invariant learning methods.

Specifically, the IRM framework approaches OOD generalization by finding an invariant representation $\varphi$,
such that there exists a classifier acting on $\varphi$ that is
simultaneously optimal in $\gS$.
Hence, IRM leads to a challenging bi-level optimization problem as
\vskip 0.1in
\begin{equation}
	\label{eq:irm}
	\min_{w,\varphi}  \ \sum_{t\in\gS}\gL_t(w\circ\varphi),                            
			\text{s.t.}
	\ w\in\argmin_{\bar{w}:\gZ\rightarrow\gY} \gL_t(\bar{w}\circ\varphi),\ \forall t\in\gS.
\end{equation}
Given the training environments $\gS$, and functional spaces $\gW$ for $w$ and $\varPhi$ for $\varphi$,
predictors $f=w\circ\varphi$ satisfying the constraint in Eq.~\ref{eq:irm} are called invariant predictors,
denoted as $\gI(\gS)$. 
When solving for invariant predictors,
characterizing $\gI(\gS)$ is particularly difficult in practice,
hence it is natural to restrict $\gW$ to be the space of linear functions on $\gZ=\R^d$~\citep{ntk}.
Furthermore, \citet{irmv1} argue that linear classifiers actually do not provide additional representation power than \emph{scalar} classifiers, i.e., $d=1,\gW=\gS=\mathbb{R}^1$. The scalar restriction elicits a practical variant as follows:
\begin{equation}
	\label{eq:irms}
		\min_{\varphi} \ \sum_{t\in\gS}\gL_t(\varphi),
		\text{s.t.}
		\ \nabla_{w|w=1}\gL_t(w\cdot\varphi)=0,\ \forall t\in\gS.
\end{equation}
Since Eq.~\ref{eq:irms} remains a constrained programming. \citet{irmv1} further introduce a soften-constrained variant, as the following
\begin{equation}
	\label{eq:irml}
	\min_{\varphi}  \sum_{t\in\gS}\gL_t(\varphi)+\lambda|\nabla_{w|w=1}\gL_t(w\cdot\varphi)|^2.
\end{equation}
If the inner optimization problem is convex, achieving feasibility is equal to the penalty term having a value of 0. Consequently, Equation~\ref{eq:irm} and Equation~\ref{eq:irml} are equivalent when we assign $\lambda=\infty.$ As proved in \citep{irmv1,rosenfeld2020risks}, the optimal solution of IRM is to extract the invariant feature $z_c$ and learn an optimal classifier on top of $z_c.$ Then, we can formulate the insight into the following Definition.
\begin{definition}
    Under the model described \ref{sec:prolem_def} and logistic loss, the optimal invariant model of IRM objective is the predictor defined by the composition of a) the featurizer which recovers the invariant features and b) the classifier which is optimal with respect to those features:
    \begin{align*}
        \phi^*(x) = [z_c, 0], \quad
        w^* = [2\mu_c/\sigma_c^2, 0]
    \end{align*}
\end{definition}

It is obvious that $\phi^*(x) = [z_c]$, and we will give the derivation how to get $w^* = [2\mu_c/\sigma_c^2].$ Notably, this derivation will also certify that the causal directions of Eq.~\ref{eq:problem} can be viewed either way.

The  probability density function  (PDF) of a  multivariate normal distribution  with mean $\mu$ and  covariance matrix  $\Sigma$ is given by:

$$
f(x; \mu, \Sigma) = \frac{1}{\sqrt{(2\pi)^k |\Sigma|}} \exp\left(-\frac{1}{2}(x-\mu)^T \Sigma^{-1} (x-\mu)\right)
$$

where $x$ is the  random variable, $k$ is the number of dimensions, and $|\Sigma|$ is the determinant of the covariance matrix.

In our case, $z_c$ is a multivariate normal distribution with mean $y\cdot \mu_c$ and covariance matrix $\sigma_c^2 I$, where $I$ is the identity matrix. Therefore, the  PDF  of $z_c$ is: 

$$
f(z_c; y\cdot \mu_c, \sigma_c^2 I) = \frac{1}{\sqrt{(2\pi)^k |\sigma_c^2 I|}} \exp\left(-\frac{1}{2}(z_c-y\cdot \mu_c)^T (\sigma_c^2 I)^{-1} (z_c-y\cdot \mu_c)\right)
$$

Since $|\sigma_c^2 I| = (\sigma_c^2)^k$, we can simplify the expression as:

\begin{align}
\label{eq:gussian}
    f(z_c; y\cdot \mu_c, \sigma_c^2 I) = \frac{1}{(\sqrt{2\pi}\sigma_c)^k} \exp\left(-\frac{1}{2\sigma_c^2}(z_c-y\cdot \mu_c)^T (z_c-y\cdot \mu_c)\right)
\end{align}
Now, we can use Bayes' rule to compute the  posterior probability  of $y$ given $z_c$:

$$
P(y=1|z_c) = \frac{P(z_c|y=1)P(y=1)}{P(z_c|y=1)P(y=1) + P(z_c|y=-1)P(y=-1)}
$$

$$
P(y=-1|z_c) = \frac{P(z_c|y=-1)P(y=-1)}{P(z_c|y=1)P(y=1) + P(z_c|y=-1)P(y=-1)}
$$

Using the PDF of $z_c$ that we derived earlier, we can compute $P(z_c|y=1)$ and $P(z_c|y=-1)$ as follows:

$$
P(z_c|y=1) = f(z_c; \mu_c, \sigma_c^2 I)
$$

$$
P(z_c|y=-1) = f(z_c; -\mu_c, \sigma_c^2 I)
$$

Substituting these expressions into the Bayes' rule equations and taking the logarithm, we get:

$$
\log\frac{P(y=1|z_c)}{P(y=-1|z_c)} = \log\frac{P(z_c|y=1)}{P(z_c|y=-1)} + \log\frac{P(y=1)}{P(y=-1)}
$$

Since we assume that the label $y \in \{1, -1\}$ is randomly sampled from the uniform distribution at the label space, we have:
$$
\log\frac{P(y=1)}{P(y=-1)} = 0,
$$
Then, substituting the expression $f(z_c; \mu_c, \sigma_c^2 I)$ for $P(z_c|y=1)$ and $f(z_c; -\mu_c, \sigma_c^2 I)$ for $P(z_c|y=-1)$, we get:

$$
P(y=1|z_c) = \frac{f(z_c; \mu_c, \sigma_c^2 I)}{f(z_c; \mu_c, \sigma_c^2 I)+f(z_c; -\mu_c, \sigma_c^2 I)}
$$
With Eq.~\ref{eq:gussian}, we have:
$$
P(y=1|z_c) = \frac{1}{1+exp(-\frac{2\mu_c z_c}{\sigma^2})}
$$
Therefore, the log-odds of $y$ is given by:
$$
\log\left(\frac{P(y=1|z_c)}{1-P(y=1|z_c)}\right)=\log\left(\frac{P(y=1|z_c)}{P(y=-1|z_c)}\right) = \frac{2\mu_c z_c}{\sigma^2}.
$$
With the above equations, we successfully show the optimal classifier under the IRM objective. In a specific domain $t,$ we can obtain a similar relationship between label $y$ and $z$ Then, we are able to give the formal version of Theorem~\ref{thm:th1} and give the corresponding proof. 

Considering $z=[z_c,z_t] \in \mathbb{R}^k$ at domain $t$, it is crucial to emphasize the significance of learning all the features from $z_c$ and $z_t$ in the following context. Our objective is to solve the standard logistic regression problem as shown in Equation~\ref{eq:logistic}.

\begin{align}
\label{eq:logistic}
        y &= \begin{cases}
        +1 &\text{w.p. } \sigma(\beta^Tz), \\
        -1 &\text{w.p. } \sigma(-\beta^Tz).
        \end{cases}
\end{align}
    
Then, we have the following theorem:
\begin{theorem}
Consider $z=[z_c,z_t] \in \mathbb{R}^k$ at domain $t$. Let $\forall S\subseteq[k],\ P(\beta_S^Tz_S \neq 0) > 0$, and assume that no feature can be written as a linear combination of the other features. Then, for any distribution $p(z)$, any classifier $f(z) = \sigma(\beta_S^T z_S)$ that uses a strict subset of the features $S\subsetneq [k]$ has strictly higher risk with logistic loss than the Bayes classifier $f^*(z) = \sigma(\beta^Tz)$. This result also holds for 0-1 loss if $\beta_{-S}^Tz_{-S}$ has greater magnitude and opposite sign of $\beta_S^Tz_S$ with non-zero probability.
\end{theorem}
\begin{proof}
    The Bayes classifier achieves the minimal expected loss for each observation $z$. Consequently, any other classifier exhibits a positive excess risk only if it disagrees with the Bayes classifier on a set of non-zero measure. Let's consider the set of values $z_{-S}$ such that $\beta_{-S}^Tz_{-S} \neq 0$. On this set, we observe the following inequality:

\begin{align*}
f^(\beta^T z) = \sigma(\beta_S^T z_S + \beta_{-S}^T z_{-S}) \neq \sigma(\beta_S^T z_S) = f(z).
\end{align*}

Since these values occur with positive probability, $f$ has a strictly higher logistic risk than $f^*$. Similarly, using the same reasoning, we can find a set of positive measure in which:

\begin{align*}
f^(\beta^T z) = \text{sign}(\beta_S^T z_S + \beta_{-S}^T z_{-S}) \neq \text{sign}(\beta_S^T z_S) = f(z).
\end{align*}

Thus, $f$ also has a strictly higher 0-1 risk.
\end{proof}

By proving the theorem above, we have established the importance of incorporating both $z_c$ and $z_t$ in order to minimize risk. Similar arguments are also provided in \citep{bui2021exploiting,rosenfeld2020risks} for conventional Domain Generalization tasks. The theorem demonstrates that the Bayes classifier, which utilizes all the features in $z$, achieves the minimal expected loss for each observation. This implies that any classifier that uses only a subset of the features, whether in logistic regression or 0-1 loss, will have a higher risk compared to the Bayes classifier.

Therefore, the theorem highlights the necessity of considering all the available features from $z_c$ and $z_t$ to ensure lower risk. Neglecting any of these features or relying on a strict subset can lead to increased risk and potentially inaccurate predictions. By leveraging the full set of features, we can capture the inherent complexity and patterns present in the data, resulting in more robust and accurate classifiers with minimized risk.

\section{Proofs of Section~\ref{sec:method}}

\label{app:proofs}

\subsection{Proofs of Theorem~\ref{thm:th2}}
\label{app:proof_stationay}
We assume that the prior distribution of latent variables $z_t$ satisfies the Markov property, indicating that each variable depends on the value of its preceding states:
\begin{equation}
p(z_t) = p(z_t|z_{<t}).
\end{equation}

\noindent The joint distribution of data and latent variables is:
\begin{equation}
\label{eq:dy_joint}
    \begin{aligned}
        & p(x_{1:T},z_c,z_{1:T}) \\
        &=  \prod_{t=1}^T p(x_t|z_c,z_t) p(z_c) p(z_t|z_{<t}) \\
        &= \prod_{t=1}^T p(z_t|z_{<t})  \prod_{i=1}^{N_t} p(z_c) p(x_{i}|z_c, z_t) \\
    \end{aligned}
\end{equation}
where $p(z_1)=p(z_1|z_0).$ Consequently, our focus lies in determining the expected Evidence Lower Bound (ELBO) for a specific domain $t$. By accomplishing this, we can readily obtain the final ELBO for $p(x_{1:T})$. Considering the introduction of two latent variables to address the two types of distribution shift, we can represent the data-generating process for a particular domain as follows:
\begin{equation}
\label{eq:single_joint}
\begin{aligned}
    p(x_t, z_c, z_t) 
    &= p(z_c) p(z_t) p(x_t|z_c, z_t) \\
    &= p(z_t) \prod_{i=1}^{N} p(z_c) p(x_{i}|z_c, z_t)
\end{aligned}
\end{equation}

\noindent Let $(z_c, z_t)$ denote the latent variables for $x_t$. Consequently, the distribution of these three latent variables can be inferred from the observable data points as $p(z_c|x_t)$ and $p(z_t|x_t)$, respectively. The joint distribution of the latent variables is given by:
\begin{equation}
\label{eq:latent_conditonal_dis}
    p(z_c, z_t|x_t) = p(z_c|x_t) p(z_t|x_t)
\end{equation}
Let $\gD_t$ be the empirical data distribution at domain $t$, assigning probability mass $1/N$ for each of the $N$ training data points in $\gD_t$. Define the aggregated posteriors as follows:
\begin{align*} \label{eq:aggregated-s}
q(z_c) & = \mathbb E_{x_t\sim \gD_t} [q(z_c|x_t)] = \frac{1}{N} \sum_{x_i\in \gD_t} q(z_c|x_i),
\\
~\\
q(z_t) & = \mathbb E_{x_t\sim \gD_t} [q(z_t|x_t)] = \frac{1}{N} \sum_{x_i\in \gD_t} q(z_t|x_i),
\\
~\\
q(s,z_t) & = \mathbb E_{x_t\sim \gD_t} [q(z_c|x_t) q(z_t|x_t)] = \frac{1}{N} \sum_{x_i\in \gD_t} q(z_c|x_i) q(z_t|x_i).
\end{align*}

\noindent With these definitions, we have
\begin{equation}
\begin{aligned}
 &\mathbb E_{x_t\sim \gD_t}[KL[q(z_c|x_t)||p(z_c)]] \\
=&\mathbb E_{x_t\sim \gD_t}\mathbb E_{q(z_c|x_t)}[\log q(z_c|x_t) - \log q(z_c) + \log q(z_c) - \log p(z_c)] \\
=&\mathbb E_{q(z_c, x_t)} \log \left[ \frac{q(z_c|x_t)}{q(z_c)} \right]  + \mathbb E_{q(z_c, x_t)} [\log q(z_c)-\log p(z_c)] \\
=& I_q (z_c;x_t) + KL [q(z_c)||p(z_c)].
\end{aligned}
\end{equation}
In other words, 
\begin{equation}
\begin{aligned}
KL [q(z_c)||p(z_c)]=\mathbb E_{x_t\sim \gD_t}[KL[q(z_c|x_t)||p(z_c)]]-I_q(z_c;x_t).
\end{aligned}
\label{eq:kl_f}
\end{equation}

\noindent Similarly, we have 
\begin{align}
\label{eq:kl_z}
    KL[q(z_t)||p(z_t)]=E_{x_t\sim \gD_t}[KL[q(z_t|x_t)||p(z_t)]] - I_q(x_t;z_t).
\end{align}

\noindent We are now ready to prove the theorem.
We derive a dataset ELBO by subtracting a different KL-divergence from the data log-likelihood:
\begin{equation}
\begin{aligned} 
&\frac{1}{N} \sum_{x_i\in D_t} \log p(x_t) =\mathbb E_{x_t \sim D_t}[\log p(x_{t})]\\
\ge& \mathbb E_{x_t \sim D_t}[\log p(x_t)-KL[q(z_c, z_t)||p(z_c, z_t|x_t)]] \\
=& \mathbb E_{x_t \sim D_t}[\mathbb E_{q(z_c, z_t|x_t)}[\log p(x_t)-(\log q(z_c, z_t)-\log p(z_c, z_t|x_t))]] \\
=& \mathbb E_{x_t \sim D_t}[\mathbb E_{q(z_c, z_t|x_t)}[\log p(x_t)-\log q(z_c, z_t)+\log p(z_c, z_t|x_t)]] \\
=& \mathbb E_{x_t \sim D_t}[\mathbb E_{q(z_c, z_t|x_t)}[\log p(x_t)-\log q(z_c, z_t)\\
&\hspace{9em} +\log p(x_t|z_c, z_t)+\log p(z_c, z_t)-\log p(x_t)]]\\
=& \mathbb E_{x_t \sim D_t}[\mathbb E_{q(z_c, z_t|x_t)}[\log p(x_t|z_c, z_t)-\log q(z_c, z_t)+\log p(z_c, z_t)]]\\
=& \mathbf{\mathbb E_{x_t \sim D_t}[\mathbb E_{q(z_c, z_t|x_t)}[\log p(x_t|z_c, z_t)]] - KL[q(z_c,z_t)||p(z_c,z_t)]}\\
=& \mathbb E_{x_t \sim D_t}[\mathbb E_{q(z_c, z_t|x_t)}[\log p(x_t|z_c, z_t)]]\\
&\hspace{5em} - \mathbb E_{x_t \sim D_t}[\mathbb E_{q(z_c, z_t|x_t)}[\log q(z_c, z_t)-\log p(z_c, z_t)]]\\
=& \mathbb E_{x_t \sim D_t}[\mathbb E_{q(z_c, z_t|x_t)}[\log p(x_t|z_c, z_t)]]\\
& -\mathbb E_{x_t \sim D_t}[\mathbb E_{q(z_c, z_t|x_t)}[\log q(z_c, z_t)-\log q(z_c)q(z_t)+\log q(z_c)q(z_t)-\log p(z_c, z_t)]]\\
=& \mathbb E_{x_t \sim D_t}[\mathbb E_{q(z_c, z_t|x_t)}[\log p(x_t|z_c, z_t)]]\\
&\hspace{5em} -\mathbb E_{x_t \sim D_t} \left[\mathbb E_{q(z_c, z_t|x_t)}\left[\log \frac{q(z_c, z_t)}{q(z_c)q(z_t)}+\log \frac{q(z_c)q(z_t)}{p(z_c, z_t)}\right]\right]\\
=& \mathbb E_{x}[\mathbb E_{q(z_c, z_t|x_t)}[\log p(x_t|z_c, z_t)]]\\
&\hspace{5em} -I_q (z_c;z_t)-\mathbb E_{x_t \sim D_t} \left[\mathbb E_{q(z_c, z_t|x_t)} \left[\log \frac{q(z_c)q(z_t)}{p(z_c)p(z_t)}\right]\right]\\
=& \mathbb E_{x_t \sim D_t}[\mathbb E_{q(z_c, z_t|x_t)}[\log p(x_t|z_c, z_t)]]-I_q (z_c;z_t)\\
&-\mathbb E_{x_t \sim D_t} \left[\mathbb E_{q(z_c, z_t|x_t)}\left[\log \frac{q(z_c)}{p(z_c)}\right]\right]-\mathbb E_{x_t \sim D_t}\left[\mathbb E_{q(z_c, z_t|x_t)}\left[\log \frac{q(z_t)}{p(z_t)}\right]\right]\\
=& \mathbb E_{x_t \sim D_t}[\mathbb E_{q(z_c, z_t|x_t)}[\log p(x_t|z_c, z_t)]]-I_q(z_c;z_t)-KL[q(z_c)||p(z_c)] -KL[q(z_t)||p(z_t)]\\
=& \mathbb E_{x_t \sim D_t}[\mathbb E_{q(z_c, z_t|x_t)}[\log p(x_t|z_c, z_t)]]-I_q(z_c;z_t)\\
&\hspace{5em} -(\mathbb E_{x_t \sim D_t}[KL[q(z_c|x_t)||p(z_c)]]-I_q(z_c;x_t))\\
&\hspace{5em} -(\mathbb E_{x_t \sim D_t}[KL[q(z_t|x_t)||p(z_t)]]-I_q(z_t;x_t))\\
=&
\mathbb E_{x_t \sim D_t} [ \mathbb E_{q(z_t, s|x_t)}[\log p(x_t|z_c, z_t)] \\
&\hspace{5em} - \mathbb E_{x_t \sim D_t} [KL[q(s|x_t)||p(z_c)]] 
- \mathbb E_{x_t \sim D_t} [KL[q(z_t|x_t)||p(z_t)]]\\
&\hspace{5em} +I_q(z_c;x_t)+I_q(z_t;x_t)-I_q(z_c;z_t).
\label{eq:single_elbo}
\end{aligned}
\end{equation}
\noindent The first inequation is due to $KL[q(z_c, z_t)||p(z_c, z_t|x_t)] \geq 0,$ and we have plugged in Eq.~\ref{eq:kl_f} and Eq.~\ref{eq:kl_z} in the third to last step. The last equation of~\ref{eq:single_elbo} is the ELBO objective at domain $t.$ Notably, the bold equation is the dataset ELBO objective at domain $t$ introduced in Theorem~\ref{thm:th2}. By utilizing the connections established in Eq.~\ref{eq:dy_joint} and \ref{eq:single_joint}, we can derive the joint distribution across all source domains.

\begin{align}
    log p(x_{1:T}) &= \sum_{t=1}^T \mathbb E_{x_t \sim D_t}[\log p(x_{t})] \nonumber \\
    \ge&\sum_{t=1}^T \mathbb E_{q(z_c, z_t|x_t)}[\log p(x_t|z_c, z_t)] - KL[q(z_c,z_t)||p(z_c,z_t)] \nonumber \\
    =& \displaystyle\sum_{t=1}^T\mathbb E_{z_c, z_t}\big[ \log p(x_t|z_c, z_t) - \big(\displaystyle\sum_{t=1}^T KL[q(z_c|x_{1:T})||p(z_c)] \nonumber \\ &+\displaystyle\sum_{t=1}^T KL[q(z_t|z_{<t}, x_{t})||p(z_{t}|p(z_{<t}))]\big) \nonumber \\  &+ \displaystyle\sum_{t=1}^T (I_q (z_c;x_x) + I_q (z_t;x_t) - I_q (z_t;z_c))\big] \nonumber \\
    \ge& \displaystyle\sum_{t=1}^T\mathbb E_{z_c, z_t}\big[ [\log p(x_t|z_c, z_t)] - (KL[q(z_c|x_{1:T})||p(z_c)] \nonumber \\ 
    &+ KL[q(z_t|z_{<t}, x_{t})||p(z_{t}|p(z_{<t})]) \nonumber \\  
    &+ (I_q (z_c;x_x) + I_q (z_t;x_t) - I_q (z_t;z_c)) \big]
\end{align}
Since each data sample is provided in the source domains and assumed to be uniformly distributed, we can omit the term $\mathbb{E}_{x_t \sim D_t}$. The final inequality is a result of Jensen's inequality. It is worth noting that the left-hand side of the first inequality corresponds to the dataset ELBO objective at domain $t$, as introduced in Theorem~\ref{thm:th2}, due to the Markov properties.

\subsection{Proof of Proposition~\ref{prop}}
We can follow the proof of Theorem 2 in \citep{DBLP:conf/icml/Qin0L22} to have a new ELBO for $p(x_{1:T},y_{1:T}).$ Specifically, we have the joint distribution of data and latent variables:
\begin{equation}
    \begin{aligned}
        & p(x_{1:T},y_{1:T},z_c,z_{1:T},w_{1:T}) \\
        &=  \prod_{t=1}^T p(x_t, y_t|z_c,z_t,w_t) p(z_c) p(z_t|z_{<t}) p(w_t|w_{<t}) \\
        &= \prod_{t=1}^T p(z_t|z_{<t}) p(w_t|w_{<t}) \prod_{i=1}^{N_t} p(z_c) p(x_{i}|z_c, z_t)p(y_{i}|z_c,z_t, w_t) \\
    \end{aligned}
\end{equation}
where $p(z_1)=p(z_1|z_0)$ and $p(w_1)=p(w_1|w_0)$. Following the spirit of \citep{DBLP:conf/icml/Qin0L22} and equation Eq.~\ref{eq:single_elbo}, we can derive the following result:

\begin{equation}\footnotesize
\begin{aligned}
\log p(x_{1:T},y_{1:T}) &\ge \mathbb{E}_{q}\log\frac{\prod_{t=1}^T p(x_t, y_t|z_c,z_t,w_t) p(z_c) p(z_t|z_{<t}) p(w_t|w_{<t})} {\prod_{t=1}^T q(z_c) q(z_t|z_{<t}) q(w_t|w_{<t},y_t)} \\
&=  \mathbb{E}_{q} \Big[ \log\frac{\prod_{t=1}^T p(z_t|z_{<t})}{\prod_{t=1}^T q(z_t|z_{<t})} + \log\frac{\prod_{t=1}^T p(w_t|w_{<t})}{\prod_{t=1}^T q(w_t|w_{<t},y_t)} + 
\log\frac{\prod_{t=1}^T p(z_c)}{\prod_{t=1}^T q(z_c)} \\
&+\log\prod_{t=1}^T p(x_t|z_c,z_t)p(y_t|z_c,z_t,w_t) \Big] \\
&= \mathbb{E}_{q} \Big[ - \sum_{t=1}^T\log\frac{q(z_t|z_{<t})}{p(z_t|z_{<t})} - \sum_{t=1}^T\log\frac{q(w_t|w_{<t},y_t)}{p(w_t|w_{<t})} - \sum_{t=1}^T\log \frac{q(z_c)}{p(z_c)} \\
&+ \sum_{t=1}^T\log p(x_t|z_c,z_t)p(y_t|z_c,z_t,w_t) \Big]\\
\end{aligned}
\end{equation}
By applying Jensen's inequality and utilizing Eq.~\ref{eq:single_elbo}, we obtain the following:
\begin{align}
    \log p(x_{1:T},y_{1:T}) &\ge \displaystyle\sum_{t=1}^T\mathbb{E}_{q} \big[\log p(x_t|z_c, z_t) \log p(y_{t}|z_c,z_{t},w_t) \nonumber \\
    &- (KL[q(z_c|x_{1:T})||p(z_c)] \nonumber \\ &+ KL[q(z_t|z_{<t}, x_{t})||p(z_{t}|p(z_{<t})]) \nonumber \\  
    &+  KL (q(w_t|w_{<t},y_t),p(w_{t}|w_{<t}))) \nonumber \\  
    &+  (I_q (z_c;x_x) + I_q (z_t;x_t) - I_q (z_t;z_c))\big] ,
\end{align}

\newpage

\section{Additional Experimental Results}
Since some of our baselines are specifically designed to generalize to the first target domains rather than the sequence target domains, we also report the results of the first target domains in Table~\ref{tab:first_domain}. The table presents the experimental results in terms of accuracy on both synthetic and real-world datasets, allowing for a comparison of different methods. It is evident from the table that our proposed method, MISTS, outperforms the other baselines significantly across all domains.

In Table~\ref{tab:first_domain}, each row corresponds to a different algorithm, while each column represents a specific dataset such as Circle, Sine, Sine-C, Elec, RMNIST, Portraits, and Caltran. Mean values are reported, along with standard deviations indicated by the $\pm$ symbol. The "Avg" column provides the average performance across all datasets.

Analyzing the different algorithms, we can observe distinct trends. Invariant Learning methods achieve moderate performance across most datasets, while data augmentation methods demonstrate competitive performance across multiple datasets. However, our proposed method, MISTS, consistently outperforms all other algorithms, achieving the highest average accuracy across all datasets.

The results undeniably highlight the superior performance of our proposed method, MISTS, making it a promising approach to address the challenges posed by Evolving Domain Generalization.

\begin{table*}[ht]
\vskip 0.15in
\begin{center}
\caption{Experimental Results (Accuracy \%) on Synthetic and Real-World Datasets across Different Methods.}
\label{tab:first_domain}
\begin{normalsize}
\begin{tabular}{c|cccccccc}
\toprule
\textbf{Algorithm} & \textbf{Circle} & \textbf{Sine}  & \textbf{Sine-C} & \textbf{Elec} & \textbf{RMNIST} &  \textbf{Portraits} & \textbf{Caltran}  & \textbf{Avg}\\
\midrule
ERM  & 56.7  $\pm$  3.2   & 73.7 $\pm$ 1.2 & 64.7 $\pm$ 4.2 & 69.7 $\pm$ 0.9 & 56.2 $\pm$ 1.0 & 74.2 $\pm$ 0.8 & 38.6 $\pm$ 3.1 &  62.0    \\
IRM   & 58.6 $\pm$ 2.5  & 68.3 $\pm$ 1.0 & 65.4 $\pm$ 3.7 &69.6 $\pm$ 1.0   & 48.2 $\pm$ 0.8 & 74.7 $\pm$ 0.9 & 47.4 $\pm$ 2.6 &  61.7  \\
Mixup  & 50.3 $\pm$ 2.1  & 64.7 $\pm$ 1.1 & 66.7 $\pm$ 2.8 &68.9 $\pm$ 0.9  & 52.3 $\pm$ 0.7 & 76.8 $\pm$ 0.8 & 55.3 $\pm$ 2.1  & 62.0   \\
LISA  & 50.8 $\pm$ 1.8   & 68.7 $\pm$ 0.8  & 68.3 $\pm$ 3.2  & 69.0 $\pm$ 1.1 & 54.1 $\pm$ 0.9 & 77.0 $\pm$ 0.9 & 56.1 $\pm$ 2.4  & 63.4  \\
CORAL  &60.4 $\pm$ 4.7  & 70.7 $\pm$ 1.5  & 67.2 $\pm$ 3.1  &69.2 $\pm$ 0.8 & 55.3 $\pm$ 0.8 & 75.1 $\pm$ 0.9 & 53.2 $\pm$ 3.5 &   64.4  \\
GroupDRO  & 63.2 $\pm$ 3.8  & 56.3 $\pm$ 0.7  & 63.4 $\pm$ 4.5  &67.5 $\pm$ 0.8  & 53.4 $\pm$ 1.0 & 76.3 $\pm$ 0.8 & 56.4 $\pm$ 3.3 & 62.4\\
\midrule
EWC  & 60.3 $\pm$ 3.6  & 85.1 $\pm$ 1.2  & 67.2 $\pm$ 3.3 & 70.1 $\pm$ 1.3  & 45.4 $\pm$ 3.7 & 77.7 $\pm$ 1.1 & 40.2 $\pm$ 3.2 & 63.7 \\
SI   & 61.1 $\pm$ 2.6 & 75.7 $\pm$ 1.5 & 68.6 $\pm$ 3.7  & 70.5 $\pm$ 1.1 &   47.1 $\pm$ 4.3 & 76.9 $\pm$ 1.0 & 43.4 $\pm$ 2.7  & 63.3  \\
\midrule
DRAIN  & 87.5 $\pm$ 3.7 & 80.5 $\pm$ 1.1 & 78.6 $\pm$ 1.1  & 70.5 $\pm$ 1.2 & 62.3 $\pm$ 1.4 & - & -    & - \\
LSSAE  & \textbf{91.5} $\pm$ \textbf{3.4} & 87.3 $\pm$ 0.9 & 63.3 $\pm$ 3.4   & 70.2 $\pm$ 0.8 & 64.1 $\pm$ 1.1 & 77.4 $\pm$ 1.5 & 59.4 $\pm$ 2.4    &  73.3 \\
\textbf{MISTS}  & 88.7 $\pm$ 3.1 & \textbf{95.2} $\pm$ \textbf{1.2} & \textbf{84.0} $\pm$ \textbf{2.0}   & \textbf{72.1} $\pm$ \textbf{1.0} & \textbf{64.3} $\pm$ \textbf{1.0} & \textbf{78.5} $\pm$ 0.9 & \textbf{62.6} $\pm$ \textbf{2.1} &  \textbf{77.9} \\

\bottomrule
\end{tabular}
\end{normalsize}
\end{center}
\vskip -0.1in
\end{table*}

\end{document}